\documentclass[11pt]{article}

\usepackage[margin=1in]{geometry}
\usepackage[T1]{fontenc}
\usepackage[utf8]{inputenc}
\usepackage{amsmath,amssymb}
\usepackage{graphicx}
\usepackage{booktabs}
\usepackage{multirow}
\usepackage{array}
\usepackage[numbers,sort&compress]{natbib}
\usepackage{hyperref}
\hypersetup{
    colorlinks=true,
    linkcolor=black,
    citecolor=black,
    urlcolor=black
}
\usepackage{xcolor}
\usepackage{microtype}
\usepackage{float}
\usepackage{amsthm}

\usepackage{xurl}

\newtheorem{Proposition}{Proposition}
\newtheorem{Corollary}{Corollary}

\title{Platonic Projection Structures: Operator-Induced Observability in Representation Learning}

\author{
\begin{minipage}{0.95\textwidth}
\centering
Kazuo Ishii$^{1,*}$, Bishnu Prasad Gautam$^{1}$, Jieling Wu$^{1}$, and Saher Javaid$^{2}$\\[0.6em]
\footnotesize
$^{1}$ Department of Applied Information Engineering, Faculty of Engineering, Suwa University of Science,\\
Chino 391-0292, Nagano, Japan\\
$^{2}$ Department of Information Engineering, Kanazawa Gakuin University,\\
Kanazawa 920-1392, Ishikawa, Japan\\
$^{*}$ Corresponding author: kishii@rs.sus.ac.jp
\end{minipage}
}

\date{}

\begin{document}

\maketitle

\begin{abstract}
We characterize observability in representation learning through Platonic Projection Structures (PPS), an operator-theoretic framework 
for analyzing representation accessibility under partial observation. 
Rather than treating observable outputs as direct reflections of latent representations, PPS models observation as a geometry induced by a self-adjoint positive semidefinite operator acting on a latent representation space.
A system is represented as a triple
$
(\mathcal{H}, \Pi, O),
$
where $\mathcal{H}$ denotes a latent representation space,
$\Pi \succeq 0$ is an observation operator, and
$O(v)=\langle v,\Pi v\rangle$
defines an induced scalar observable.
The framework characterizes observability through the quotient geometry
$
\mathcal{H}/\ker(\Pi),
$
which represents equivalence classes of latent states that are indistinguishable under observation.
From this perspective, observable behavior is governed not by latent representations themselves, but by the geometry induced through the observation operator.
We show that both quantum measurement and representation inference under linear observation models can be formulated within this common 
operator-theoretic structure while differing in the algebraic properties of their observation operators. 
Within this perspective, quantum measurement serves primarily as a mathematically canonical example of projection-mediated observability.
The correspondence developed in PPS is therefore structural rather than physical.
Within the same framework, representation transfer and knowledge distillation can be interpreted as approximate preservation of observable geometry through the intertwining condition
$
\Phi \Pi_T \approx \Pi_S \Phi.
$
PPS further reveals a structural limitation of output-based interpretability:
latent components contained in $\ker(\Pi)$ are fundamentally inaccessible from observables generated through the induced observation process.
Accordingly, attribution and explanation methods inherit intrinsic constraints imposed by the observation geometry itself.
We provide controlled empirical validations demonstrating kernel-invariant observability, projection-induced attribution gaps, and rank-controlled observable geometry in latent representation spaces.
Overall, PPS provides a mathematically explicit characterization of observability through operator-induced quotient geometry, offering a unified perspective on representation accessibility, interpretability, and projection-mediated inference.
\end{abstract}


\section{Introduction}
{\sloppy
Modern machine learning systems are increasingly interpreted through observable quantities 
such as output distributions, attribution maps, representations, and feature 
activations~\cite{bengio_rep,doshi,rudin}.
\par}
At the same time, recent developments in information-theoretic analysis
have highlighted that observable outputs often provide only partial access
to the underlying latent structure of a system
~\cite{tishby_ib,shwartz_ziv,alei_vib,mine}.
Despite substantial progress in explainable artificial intelligence (XAI),
representation analysis, and information bottleneck \mbox{approaches
~\cite{tishby_ib,shwartz_ziv,alei_vib,mine,bengio_rep,doshi,rudin},}
the notion of \emph{observability} itself remains largely implicit rather than explicitly 
formalized in modern learning systems.
In many existing formulations, observable outputs are implicitly treated as sufficiently informative reflections of latent representations. 
However, practical learning systems frequently operate under restricted observation geometries in which only selected latent directions contribute to observable behavior.
This paper introduces \emph{Platonic Projection Structures} (PPS), an operator-theoretic framework for analyzing projection-mediated observability in representation learning systems~\cite{halmos,reed_simon,bhatia}. 
Rather than treating observation as direct access to latent states, PPS models observable quantities as induced through a self-adjoint positive semidefinite operator acting on a latent Hilbert space~\cite{halmos,bhatia}.
Within PPS, a system is represented as
\begin{equation}
(\mathcal{H}, \Pi, O),
\label{eq:pps_triple}
\end{equation}
where $\mathcal{H}$ denotes a latent representation space, $\Pi \succeq 0$ is an observation operator, and the observable quantity is defined by
\begin{equation}
O(v)=\langle v,\Pi v\rangle.
\label{eq:observable}
\end{equation}
The primary observable object in PPS is the projected representation
$\Pi v$, which determines observable equivalence classes.
The scalar quantity $O(v)=\langle v,\Pi v\rangle$ is introduced as an
operator-induced energy functional that summarizes certain observable
aspects of the representation, but does not by itself characterize the
full observable geometry.
Observable accessibility is determined by the projection operator through
the equivalence relation induced by $\Pi$ and the resulting quotient
structure $\mathcal{H}/\ker(\Pi)$.
The central principle of PPS is that observability is governed not by latent representations themselves, 
but by the geometry induced through the observation operator. 
These assumptions are standard in operator theory and positive semidefinite analysis~\cite{reed_simon,bhatia}.

The operator $\Pi$ induces the equivalence relation
\begin{equation}
v_1 \sim_\Pi v_2
\quad\Longleftrightarrow\quad
v_1-v_2\in\ker(\Pi),
\label{eq:equivalence}
\end{equation}
yielding the quotient structure
\begin{equation}
\mathcal{H}/\ker(\Pi),
\label{eq:quotient}
\end{equation}
which represents the effective observable space.
Under this formulation, latent components contained in $\ker(\Pi)$ are structurally inaccessible from observables generated through Equation~\eqref{eq:observable}.
This viewpoint leads naturally to a spectral characterization of observability through the eigenspectrum of the induced observation operator~\cite{scholkopf_smola,aronszajn,jolliffe,spectral_ml,ng_spectral}.
The eigenspectrum of $\Pi$ determines which latent directions are strongly observable, weakly observable, or entirely suppressed. 
Consequently, within the PPS formulation, the effective observable dimension is characterized by the rank and 
spectral structure of the observation operator~\cite{spectral_ml,scholkopf_smola,amari}.
The PPS framework provides a unified abstraction for projection-mediated observation across multiple domains. This perspective is related to operator-theoretic formulations appearing in quantum mechanics, kernel methods, and information geometry~\cite{vonneumann,nielsen2000,schuld,biamonte,density_ml,quantum_kernel,amari}. In quantum systems, idealized measurements correspond to orthogonal projections satisfying~\cite{vonneumann,nielsen2000}
\begin{equation}
\Pi^2=\Pi,
\label{eq:projection}
\end{equation}
whereas deep learning systems generally induce non-idempotent positive semidefinite operators under linear output mappings.
Despite these differences, both settings share the same abstract operator-theoretic structure: observable quantities are generated through operator-induced structure rather than through direct access to latent states.
Within this perspective, representation transfer and knowledge distillation can also be interpreted geometrically. 
Within PPS, projection-mediated observability is treated as an explicit mathematical structure rather than as a purely metaphorical analogy between domains.

Given teacher and student observation operators $\Pi_T$ and $\Pi_S$, a transfer map
\[
\Phi:\mathcal{H}_T\rightarrow\mathcal{H}_S
\]
approximately preserves observable geometry when
\begin{equation}
\Phi\Pi_T \approx \Pi_S\Phi.
\label{eq:intertwining}
\end{equation}
In the present formulation, $\Phi$ is assumed to be a linear transfer map acting on the latent representation space. 
This formulation is motivated by approximate intertwining relations commonly appearing in operator-theoretic alignment and representation-transfer settings. 
Equation~\eqref{eq:intertwining} interprets distillation as approximate consistency between induced observation 
geometries, rather than solely as output-level distribution matching~\cite{hinton2015,fitnets,rkd,crd,attention_transfer}.
A central implication of PPS concerns the structural limits of output-based interpretability. 
Since observables generated through Equation~\eqref{eq:observable} are invariant under perturbations contained in $\ker(\Pi)$, they depend only on the equivalence classes induced by Equation~\eqref{eq:equivalence} and represented by the quotient structure Equation~\eqref{eq:quotient}.
Latent components lying in $\ker(\Pi)$ cannot, in general, be recovered from observables alone.
From this perspective, limitations of interpretability are not merely algorithmic, but geometric and operator-induced.

The PPS framework is conceptually related to several existing research directions.
Information bottleneck formulations characterize information compression through mutual-information constraints~\cite{shannon,tishby_ib,shwartz_ziv,alei_vib,mine}, whereas PPS characterizes observability through operator-induced geometry on latent spaces.
Similarly, concept bottleneck models and representation-constrained learning methods
may be interpreted as imposing explicit structural constraints on induced observable subspaces
~\cite{cbm,tcav,irm,causal_rep}. PPS also exhibits structural parallels with observability theory in control systems~\cite{kalman}, 
although the present framework focuses on projection-induced accessibility in representation spaces rather than 
dynamical state reconstruction.

To evaluate the PPS formulation empirically, we construct controlled latent-space experiments that directly examine projection-induced observability. The experiments verify three central predictions of the framework:
(i) observables generated through Equation~\eqref{eq:observable} remain invariant under perturbations confined to $\ker(\Pi)$,
(ii) the rank structure of $\Pi$ governs effective observable dimensionality and predictive behavior, and
(iii) attribution methods inherit structural limitations induced by the observation geometry.
These experiments are designed primarily as theory-validation studies intended to isolate projection-mediated 
observability effects rather than as benchmark-oriented evaluations of predictive performance. Accordingly, 
the goal of the empirical analysis is not to establish state-of-the-art performance on large-scale datasets, 
but to directly test whether the structural predictions of PPS emerge under controlled experimental conditions.

The primary contribution of this work is not the introduction of new operator-theoretic machinery. 
While the mathematical ingredients of PPS, including positive semidefinite operators, kernels, quotient 
spaces, and spectral decompositions, are classical, the central novelty lies in identifying and formalizing 
an operator-induced observability structure that has not been explicitly characterized in existing 
representation-learning frameworks.

The central novelty lies in explicitly formulating 
observability as an operator-induced equivalence structure governing representation accessibility. 
In particular, PPS identifies the quotient geometry $\mathcal{H}/\ker(\Pi)$ as the fundamental object 
underlying observable representations and provides a unified geometric framework connecting observability, 
representation transfer, and interpretability.

The resulting framework provides a mathematically explicit perspective on observability across projection-mediated representation systems. The main contributions of this paper are summarized as follows:
\begin{enumerate}
\item We introduce PPS, an operator-theoretic framework that formalizes observability in representation learning as an operator-induced geometric structure.
\item We identify operator-induced observability equivalence classes as a fundamental object governing 
representation accessibility and formalize them through the quotient geometry $\mathcal{H}/\ker(\Pi)$.
\item We provide a unified geometric interpretation of representation transfer, knowledge distillation, 
and output-based interpretability through operator-induced observability.
\item We establish the role of the spectral structure of positive semidefinite observation operators in determining effective observable dimensions and accessibility.
\item We provide empirical analyses validating key theoretical predictions of PPS, including kernel-invariant observability, projection-induced attribution gaps, and rank-controlled observable geometry.
\end{enumerate}

\section{Operator-Theoretic Framework of Observability}

This section introduces Platonic Projection Structures (PPS) as an operator-theoretic framework for analyzing projection-mediated observability in representation learning systems. 
Rather than treating observable quantities as direct reflections of latent states, PPS models observation as a geometric process induced through a self-adjoint positive semidefinite operator acting on a latent Hilbert space.

\subsection{Platonic Projection Structures}

Let $\mathcal{H}$ denote a real or complex Hilbert space of latent representations. 
Within the PPS framework introduced in Equation~\eqref{eq:pps_triple}, observability is induced through an observation operator
$
\Pi:\mathcal{H}\rightarrow\mathcal{H},
$
which generates observable quantities through the quadratic functional defined in Equation~\eqref{eq:observable}.
We distinguish two regimes of observation:
\begin{enumerate}
\item Strict projection operators satisfying Equation~\eqref{eq:projection};
\item Generalized positive semidefinite operators satisfying
$
\Pi \succeq 0.
$
\end{enumerate}
Idealized quantum measurements naturally belong to the first regime~\cite{vonneumann,nielsen2000}, whereas representation learning systems are generally associated with the second.
Despite this distinction, both settings share the same structural mechanism:
observable quantities are generated through operator-induced structure rather than through direct access to 
latent states.
The observation operator is assumed to satisfy
\[
\Pi^\dagger=\Pi,
\qquad
\langle v,\Pi v\rangle \geq 0
\quad
\forall v\in\mathcal{H},
\]
ensuring that the observable functional in Equation~\eqref{eq:observable} is well-defined and non-negative.
Under PPS, the operator $\Pi$ becomes the primitive object governing observability.
Observable structure is therefore determined not by latent representations themselves, but by the geometry induced through the observation operator.

\subsection{Observable Geometry and Quotient Structure}

Throughout this paper, the term ``observable geometry'' refers to the quotient-semimetric structure induced by $\Pi$ through the equivalence relation in Equation~\eqref{eq:equivalence} and the semi-inner product defined in Equation~\eqref{eq:semi_inner}.
A central consequence of PPS is that observability is defined only up to the equivalence relation introduced in Equation~\eqref{eq:equivalence}. 
Two latent states differing only by a component contained in $\ker(\Pi)$ are observationally indistinguishable under Equation~\eqref{eq:observable}.
Accordingly, the effective observable space is not the ambient latent space
$\mathcal H$ itself, but the quotient space
$\mathcal H/\ker(\Pi)$ defined in Equation~\eqref{eq:quotient},
a standard construction in operator theory~\cite{halmos,reed_simon}.
This quotient space represents equivalence classes of latent states
that generate identical observables under the induced observation process.
The operator $\Pi$ further induces the bilinear form
\begin{equation}
\langle v,w\rangle_\Pi
:=
\langle v,\Pi w\rangle,
\label{eq:semi_inner}
\end{equation}
which defines a semi-inner product structure on $\mathcal{H}$.
Since directions contained in $\ker(\Pi)$ contribute no observable signal,
the form in Equation~\eqref{eq:semi_inner} becomes non-degenerate only after
quotienting by $\ker(\Pi)$.
The quotient geometry in Equation~\eqref{eq:quotient} therefore defines the
effective observable geometry induced by the observation operator
~\cite{halmos,reed_simon}.
Under this viewpoint, observability becomes a structural property of induced operator geometry rather than a property of individual latent vectors.

\subsection{Spectral Structure of Observability}

Building on classical spectral methods
~\cite{scholkopf_smola,aronszajn,jolliffe,spectral_ml,ng_spectral},
the PPS framework naturally admits a spectral interpretation of observability.
Let
\begin{equation}
\Pi
=
\sum_i \lambda_i e_i e_i^\ast
\label{eq:spectral}
\end{equation}
denote the spectral decomposition of the observation operator.
The eigenvalues $\lambda_i$ determine the degree of observability associated with each latent direction:

\[
\lambda_i \gg 0
\Rightarrow
\text{strongly observable direction},
\]

\[
\lambda_i \approx 0
\Rightarrow
\text{weakly observable direction},
\]

\[
\lambda_i = 0
\Rightarrow
e_i \in \ker(\Pi)
\Rightarrow
\text{structurally unobservable direction}.
\]
Accordingly, the effective observable dimension is governed by
\begin{equation}
\mathrm{rank}(\Pi).
\label{eq:observable_rank}
\end{equation}
Equation~\eqref{eq:observable_rank} implies that observable dimensionality depends not on the ambient latent dimension itself, but on the spectral structure of the observation operator~\cite{lowrank,spectral_ml}.
From an information-theoretic perspective, the eigenspectrum in Equation~\eqref{eq:spectral} characterizes the relative accessibility of latent information under the induced observation \mbox{process~\cite{shannon,tishby_ib,amari}.}
Large eigenvalues correspond to latent directions strongly represented in observable quantities, whereas directions associated with vanishing eigenvalues become inaccessible through Equation~\eqref{eq:observable}.

\subsection{Structural Factorization of Observation}
The observation process is induced by PPS factors naturally through the quotient geometry defined in 
Equation~\eqref{eq:quotient}~\cite{halmos,reed_simon}.
Specifically, the observable functional in Equation~\eqref{eq:observable} admits the factorization
\begin{equation}
\mathcal{H}
\xrightarrow{\, q \,}
\mathcal{H}/\ker(\Pi)
\xrightarrow{\, \widetilde{O} \,}
\mathbb{R},
\label{eq:factorization}
\end{equation}
where $q$ denotes the canonical quotient map.
Equation~\eqref{eq:factorization} formalizes a central implication of PPS:
observable quantities do not uniquely determine latent states, but only equivalence classes induced by the observation geometry.
Consequently, latent components lying in $\ker(\Pi)$ are structurally inaccessible from observables generated through Equation~\eqref{eq:observable}. 
This limitation is therefore geometric rather than merely algorithmic.
From an information-geometric viewpoint, the quotient structure in Equation~\eqref{eq:quotient} may be interpreted as the effective observable information geometry induced by the observation operator~\cite{shap,gradcam,ig,rudin}.
The observation process therefore acts as a structured projection mechanism that compresses latent information according to the spectral accessibility encoded in Equation~\eqref{eq:spectral}~\cite{tishby_ib,alei_vib,mine}.
Overall, PPS provides a unified operator-theoretic framework in which observability, representation accessibility, and projection-induced information loss are described through induced observable geometry~\cite{bhatia,amari,bengio_rep}.

\subsection{Mathematical Setting and Assumptions}
The PPS framework is formulated in a Hilbert space setting in order to
provide a general operator-theoretic description of observability.
Throughout the theoretical development, we assume that the observation
operator
$
\Pi:\mathcal H \rightarrow \mathcal H
$
is a bounded self-adjoint positive semidefinite operator.
For the empirical studies presented in this paper, all latent
representation spaces are finite-dimensional Euclidean spaces
$\mathbb R^n$.
In this setting, $\Pi$ is represented by a symmetric positive
semidefinite matrix, and its spectral decomposition is finite and exact.
Consequently, the quotient geometry
$\mathcal H/\ker(\Pi)$ can be interpreted through standard finite-dimensional linear algebra.
The Hilbert space formulation is retained primarily to emphasize the generality of the operator-theoretic perspective rather than to require infinite-dimensional analysis.

\section{Projection-Mediated Observation Across Systems}

This section shows that both quantum measurement and deep learning inference can be interpreted within the PPS framework as instances of projection-mediated observation~\cite{vonneumann,nielsen2000,biamonte,schuld}. 
The correspondence is structural rather than physical: in both settings, observable quantities are generated through operator-induced geometry acting on latent representation spaces~\cite{reed_simon,bhatia,amari}.

\subsection{Quantum Measurement as Projection-Induced Observation}

Quantum measurement provides a canonical example of projection-mediated observability~\cite{vonneumann,nielsen2000}. 
Throughout this paper, $\Pi$ denotes a general observation operator in PPS, 
whereas $P_m$ refers specifically to a quantum measurement projector.
A quantum system is represented by a normalized state vector
$[
|\psi\rangle \in \mathcal{H}_Q,
\langle\psi|\psi\rangle = 1,
$]
where $\mathcal{H}_Q$ denotes a Hilbert space of admissible states.
Measurement is described by a Hermitian operator admitting the spectral decomposition~\cite{vonneumann,reed_simon}
\begin{equation}
M
=
\sum_m m P_m,
\label{eq:measurement}
\end{equation}
where each $P_m$ is an orthogonal projection operator satisfying Equation~\eqref{eq:projection}.
The probability of observing outcome $m$ is given by the Born rule~\cite{vonneumann,nielsen2000}
\begin{equation}
p(m)
=
\langle\psi|P_m|\psi\rangle.
\label{eq:born}
\end{equation}
Equation~\eqref{eq:born} is structurally equivalent to the PPS observable functional defined in \mbox{Equation~\eqref{eq:observable}} under the identification
\[
v=|\psi\rangle,
\qquad
\Pi=P_m.
\]
From the PPS viewpoint, quantum measurement can therefore be interpreted as a projection-mediated observation process:
\[
|\psi\rangle
\rightarrow
P_m
\rightarrow
p(m).
\]
This formulation emphasizes that measurement accesses only the projected component of the quantum state.
Information contained in complementary directions remains inaccessible under the induced observation process.

\subsection{Deep Learning Inference as Operator-Induced Observation}

Deep learning models equipped with linear output mappings can similarly be interpreted within the 
PPS framework as operator-mediated observation processes~\cite{bengio_rep}.
Let
\begin{equation}
f_\theta
=
f_L\circ f_{L-1}\circ\cdots\circ f_1
\label{eq:network}
\end{equation}
denote a deep neural network, and let
\[
z=f_\theta(x)\in\mathcal{H}_L
\]
represent the final latent representation.
In standard architectures, outputs are generated through a linear readout
\begin{equation}
y=Wz,
\label{eq:linear_readout}
\end{equation}
where $W$ maps latent representations to output space.
The corresponding quadratic observable becomes~\cite{bhatia}
\begin{equation}
\|y\|^2
=
\langle z,W^\top W z\rangle.
\label{eq:dl_observable}
\end{equation}
Equation~\eqref{eq:dl_observable} naturally defines the positive semidefinite observation operator~\cite{bhatia}
\begin{equation}
\Pi:=W^\top W.
\label{eq:wtw}
\end{equation}
Within architectures employing linear output mappings, Equation~\eqref{eq:wtw} shows that observable behavior 
is governed not solely by latent representations themselves, but also by the operator-induced geometry 
through which latent information becomes accessible~\cite{scholkopf_smola,spectral_ml,amari}.
The quadratic observable defined above should be understood as a canonical representation of the geometry 
induced by the linear readout operator.
It does not aim to fully characterize nonlinear output mappings such as softmax-based classification.
Such nonlinear observables lie outside the scope of the present formulation and constitute an important direction 
for future extensions.

Under PPS, deep learning inference can therefore be interpreted as
\[
z
\rightarrow
\Pi
\rightarrow
\|y\|^2,
\]
\textls[-15]{which is structurally analogous to the projection-mediated process described by \mbox{Equation~\eqref{eq:born}.}}
Importantly, PPS does not claim that neural networks are fundamentally quadratic systems.
Rather, Equation~\eqref{eq:wtw} identifies a canonical observable geometry induced by the output map.
Directions contained in $\ker(\Pi)$ remain inaccessible from observables generated through Equation~\eqref{eq:dl_observable}. 
Consequently, latent representations are observable only up to the quotient geometry defined in Equation~\eqref{eq:quotient}.
\subsection{Structural Correspondence Between Quantum and Learning Systems}

Both quantum systems and learning systems equipped with linear observation operators
can be viewed through a common abstract structure
~\cite{vonneumann,biamonte,schuld,density_ml}
consisting of
\begin{enumerate}
\item A latent Hilbert space $\mathcal{H}$;
\item An observation operator $\Pi$;
\item An induced observable generated through Equation~\eqref{eq:observable}.
\end{enumerate}
Under this perspective, the two domains admit the formally parallel structure
\begin{align}
\text{Quantum system:}\quad&
\mathcal{H}_Q,\;
|\psi\rangle,\;
P_m,
\;
p(m)=\langle\psi|P_m|\psi\rangle,
\nonumber
\\
\text{Deep learning system:}\quad&
\mathcal{H}_L,\;
z,\;
W^\top W,
\;
\|y\|^2=\langle z,W^\top W z\rangle.
\nonumber
\end{align}
In both cases, observable quantities are induced through self-adjoint positive semidefinite operators acting on latent spaces~\cite{bhatia,reed_simon}.
The distinction lies not in the existence of projection-mediated observability itself, but in the algebraic structure of the corresponding operators.
Quantum systems typically involve orthogonal projections satisfying \mbox{Equation~\eqref{eq:projection}, }
whereas learning systems with linear output mappings induce non-idempotent positive semidefinite operators 
such as Equation~\eqref{eq:wtw}.
This correspondence should therefore be understood as operator-theoretic rather than physical.
\textls[-15]{PPS does not identify neural inference with quantum dynamics, but instead provides a unified geometric 
language for describing observability across projection-mediated systems~\cite{biamonte,schuld,quantum_kernel}. 
Figure~\ref{fig:pps_unified} summarizes the resulting correspondence between these systems.
PPS should therefore be understood as a simplified geometric characterization of observability under 
linear observation operators rather than as a complete description of modern neural-network architectures.}
\begin{figure}[H]

    \includegraphics[width=\linewidth]{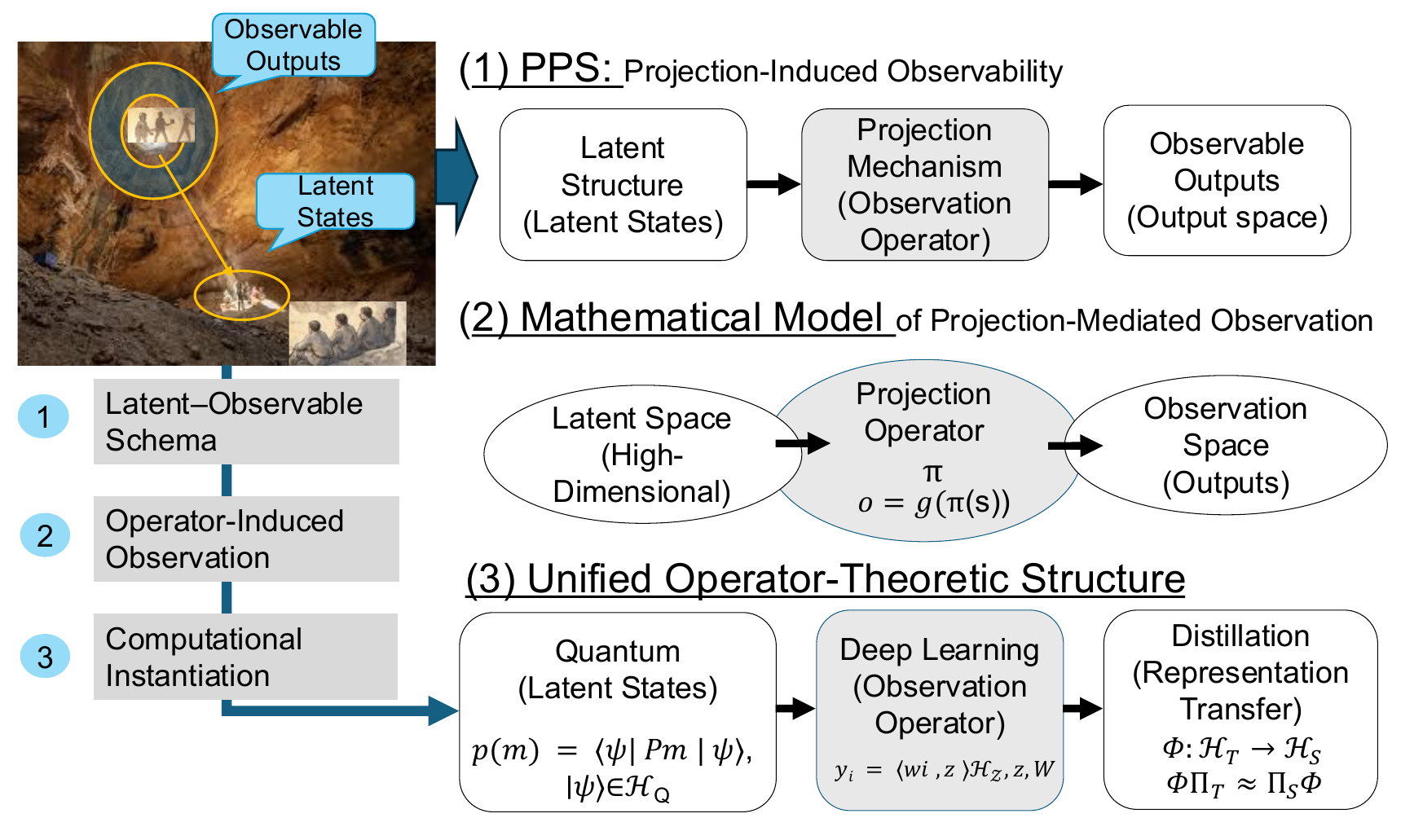}
    \caption{
    {PPS framework as projection-induced observation geometry.
}
Schematic illustration of projection-mediated observability under PPS. A latent state $v\in\mathcal{H}$ is mapped through a self-adjoint positive semidefinite operator $\Pi$ to an observable quantity generated through Equation~\eqref{eq:observable}.
    Quantum systems correspond to projection operators satisfying Equation~\eqref{eq:projection}, whereas deep learning systems generally induce non-idempotent positive semidefinite operators such as Equation~\eqref{eq:wtw}.
    }
    \label{fig:pps_unified}
\end{figure}

\section{Information Geometry and Representation Transfer}

Building on information theory, information bottleneck theory,
information geometry, and representation transfer learning
~\cite{shannon,tishby_ib,amari,hinton2015},
this section develops an information-theoretic interpretation of PPS
and introduces a geometric formulation of representation transfer.
Within PPS, observable information is determined not only by latent
representations themselves, but also by the observation geometry through
which latent directions become accessible. Representation transfer can
therefore be interpreted as approximate preservation of this induced
observable geometry between latent spaces.

\subsection{Information-Theoretic Interpretation of PPS}
Motivated by information theory, information bottleneck theory, and information geometry~\cite{shannon,tishby_ib,amari}, the PPS framework admits a natural interpretation from the perspective of information accessibility.
Let
$
v\in\mathcal{H}
$
be a latent representation, and consider observables generated through the quadratic functional defined in Equation~\eqref{eq:observable}. The observation operator determines which latent directions contribute to observable quantities.
The spectral decomposition introduced in Equation~\eqref{eq:spectral} therefore characterizes the relative accessibility of latent information~\cite{tishby_ib,amari}. 
Large eigenvalues correspond to strongly observable latent directions, whereas vanishing eigenvalues identify directions that are completely unobservable. 
Accordingly, the effective observable dimension is governed by Equation~\eqref{eq:observable_rank}, rather than by the ambient latent dimensionality itself.
Observable quantities generated through Equation~\eqref{eq:observable}
depend only on information contained within the effective observable quotient
defined in Equation~\eqref{eq:quotient}.
\subsection{Knowledge Distillation as Operator Alignment}

Knowledge distillation can be interpreted as a geometric alignment problem between operator-induced observable structures~\cite{hinton2015,rkd,crd,attention_transfer}.
Standard knowledge distillation is commonly formulated by minimizing the divergence between teacher and student output distributions~\cite{hinton2015}:
\begin{equation}
L_{\mathrm{KD}}
=
D_{\mathrm{KL}}(p_T\parallel p_S),
\label{eq:kd_loss}
\end{equation}
where
\[
p_T=\sigma(W_T z_T),
\qquad
p_S=\sigma(W_S z_S).
\]
From the PPS viewpoint, Equation~\eqref{eq:kd_loss} aligns observable distributions in output space while leaving the geometry of the underlying observation operators implicit.
Several extensions, including FitNets~\cite{fitnets}, RKD~\cite{rkd}, CRD~\cite{crd}, and attention transfer~\cite{attention_transfer}, introduce increasingly structured alignment constraints between latent representations.
Within PPS, these approaches may be viewed as introducing progressively richer forms of alignment that can be analyzed through the operator-alignment perspective.
We associate the following positive semidefinite operators with teacher and student models:
\begin{equation}
\Pi_T=W_T^\top W_T,
\qquad
\Pi_S=W_S^\top W_S.
\label{eq:teacher_student_ops}
\end{equation}
The corresponding quadratic observables become
\begin{equation}
O_T(z_T)
=
\langle z_T,\Pi_T z_T\rangle,
\qquad
O_S(z_S)
=
\langle z_S,\Pi_S z_S\rangle.
\label{eq:teacher_student_observables}
\end{equation}
Equation~\eqref{eq:teacher_student_ops} defines the induced observation geometries of the teacher and \mbox{student systems.}

\subsection{Operator Consistency and Intertwining Relations}

When the teacher and student latent spaces differ, we introduce a linear map
\begin{equation}
\Phi:\mathcal{H}_T\rightarrow\mathcal{H}_S.
\label{eq:interspace_map}
\end{equation}
A natural notion of structural compatibility between teacher and student observation geometries is expressed through the approximate intertwining condition introduced previously in Equation~\eqref{eq:intertwining}.
Equation~\eqref{eq:intertwining} states that projecting before or after the transfer map yields approximately consistent observable structure.
This condition expresses approximate preservation of observable structure under representation transfer.
The corresponding commutative structure is represented by
\[
\begin{array}{ccc}
\mathcal{H}_T
& \xrightarrow{\ \Pi_T\ } &
\mathcal{H}_T
\\
\downarrow{\Phi}
&
&
\downarrow{\Phi}
\\
\mathcal{H}_S
& \xrightarrow{\ \Pi_S\ } &
\mathcal{H}_S
\end{array}
\]
Under exact commutativity, observable quotient structures induced by Equation~\eqref{eq:quotient} are preserved under the transfer map.
Approximate commutativity therefore provides a geometric notion of observable consistency between latent representation systems.
To empirically examine this interpretation, we compare standard KL-based distillation with a PPS-regularized variant explicitly incorporating the operator-consistency objective
\begin{equation}
L_{\mathrm{PPS}}
=
\|\Phi\Pi_T-\Pi_S\Phi\|_F^2.
\label{eq:pps_loss}
\end{equation}

\subsection{Experimental Setup}
To evaluate the PPS formulation in a practical representation-learning setting, we conducted a 
knowledge-distillation experiment on the CIFAR-10 dataset.
The teacher model was a ResNet-18 classifier, while the student model consisted of a lightweight 
convolutional neural network with two convolutional blocks followed by two fully connected layers. 
Both models were trained on CIFAR-10 using the Adam optimizer with a learning rate of $10^{-3}$ and a 
batch size of 128.
For reproducibility, all experiments were performed using a fixed random seed of 42. 
The teacher model was trained for 10 epochs. Student models were trained under two conditions:
(1) standard knowledge distillation using KL-divergence and (2) PPS-aware distillation using KL-divergence together with the PPS commutativity regularization term defined in Equation~\eqref{eq:pps_loss}.
The distillation temperature was set to $T=4.0$, and the PPS regularization coefficient was set to $\lambda=0.5$.
The observation operators were constructed from the final linear layers of the teacher and student \mbox{networks as}
\[
\Pi_T = W_T^\top W_T,
\qquad
\Pi_S = W_S^\top W_S ,
\]
where $W_T$ and $W_S$ denote the corresponding output-layer weight matrices.
The transfer map $\Phi:\mathcal{H}_T \rightarrow \mathcal{H}_S$ was implemented as a bias-free trainable 
linear transformation with $\Phi \in \mathbb{R}^{d_S \times d_T}$,
where $d_T$ and $d_S$ denote the dimensions of the teacher and student penultimate feature spaces, respectively.
For empirical analyses involving spectral quantities, numerical rank was computed using an eigenvalue
threshold of $\varepsilon = 10^{-6}$, with eigenvalues below this threshold treated as numerically zero.
The experiment was introduced specifically to evaluate whether the PPS formulation remains meaningful 
when applied to learned neural-network representations rather than synthetic latent-space constructions.

The primary objective of this experiment is not to establish superior 
predictive performance, but to evaluate whether operator-level consistency 
can be introduced as an explicit optimization target within 
representation-transfer frameworks. 
Whether reduced operator inconsistency leads to improved transferability, 
robustness, interpretability, or performance on unseen tasks remains an 
important open question for future investigation.

Figure~\ref{fig:projection_noncommutativity} summarizes the comparison between standard KL-based distillation and PPS-regularized distillation.
Panel (a) shows the evolution of the commutativity gap induced by Equation~\eqref{eq:pps_loss}, panel (b) reports predictive accuracy, and panel (c) shows the final operator inconsistency after training.
The PPS-regularized formulation consistently reduces operator inconsistency between teacher and student systems while maintaining comparable predictive performance~\cite{hinton2015,rkd,crd}.

\begin{figure}[H]

\includegraphics[width=.7\linewidth]{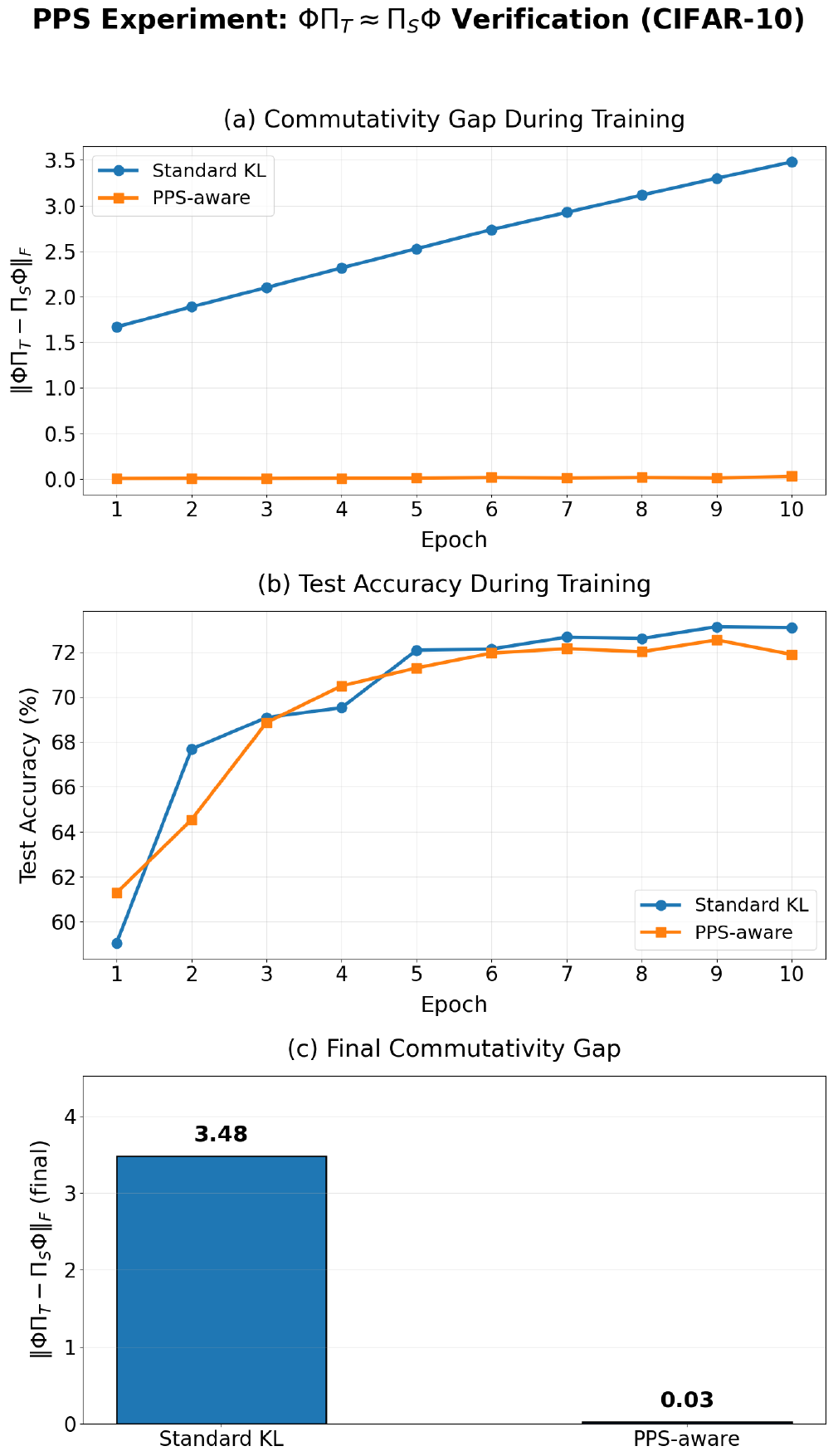}
\caption{
{Operator-consistent
 knowledge distillation under PPS.}
Comparison between standard KL-based distillation and PPS-regularized distillation on CIFAR-10.
The PPS formulation explicitly introduces the operator-consistency objective defined in Equation~\eqref{eq:pps_loss}.
(\textbf{a}) Evolution of the commutativity gap during training.
(\textbf{b}) Test accuracy over epochs.
(\textbf{c}) Final operator inconsistency after training.
The results indicate that operator-level consistency can be introduced as an explicit optimization objective without degrading predictive performance in this setting.
}
\label{fig:projection_noncommutativity}
\end{figure}

These results demonstrate that operator-level consistency can be measured
and explicitly optimized within a standard distillation framework while
maintaining comparable predictive performance. 
The findings suggest that observable geometry may provide a complementary 
perspective on representation transfer beyond output-level distribution matching, 
thereby motivating future investigation of transferability, robustness, and 
interpretability from an operator-theoretic viewpoint.
Within PPS, the primary object of analysis therefore becomes the geometry induced by the operators themselves rather than the output distributions alone. 
The code used in this study is available from the corresponding author upon reasonable request.

Rather than claiming universality, this experiment serves as a controlled illustration that operator-level geometric consistency can be introduced as an explicit and measurable quantity within standard representation-transfer pipelines. 

\section{Structural Limits of Output-Based Interpretability}

This section analyzes the structural limits of output-based interpretability from the PPS perspective~\cite{doshi,rudin,amari}.
The central implication is that explanation methods operating on observable outputs cannot, in general, access the full latent representation space.
Instead, they inherit the geometric constraints imposed by the observation operator.
Under PPS, observables generated through Equation~\eqref{eq:observable} depend only on the operator-induced observable geometry.
Consequently, latent components contained in $\ker(\Pi)$ are structurally inaccessible to explanation methods derived solely from outputs.
This limitation is not a failure of a particular attribution algorithm, but a consequence of projection-mediated observability itself.

\subsection{Projection-Induced Observability Limits}

To make this limitation explicit, we consider the strict projection regime introduced in Equation~\eqref{eq:projection}~\cite{vonneumann,halmos}.
Let $P:\mathcal{H}\rightarrow\mathcal{H}$ be an orthogonal projection operator.
For a latent representation $x\in\mathcal{H}$, the projection-induced decomposition is given by
\begin{equation}
x
=
P(x)
+
\bigl(x-P(x)\bigr),
\label{eq:projection_decomposition}
\end{equation}
where $P(x)$ is the observable component and $x-P(x)$ lies in the complementary unobservable component.
Equation~\eqref{eq:projection_decomposition} formalizes a key point:
interpretability is not an intrinsic property of the latent representation $x$ alone, but of the observation operator acting on it.
In this regime, output-based explanations can access only the component preserved by $P$.
Latent directions outside $\mathrm{range}(P)$ do not contribute to observable outputs and therefore cannot be recovered by explanation methods that operate only after projection.
This projection-induced decomposition is illustrated conceptually in Figure~\ref{fig:interpretability_projection}.

\subsection{Observational Equivalence Under Projection}

The quotient geometry introduced in Equation~\eqref{eq:quotient} implies that distinct latent states may be 
observationally indistinguishable.
The following proposition formalizes this observational equivalence.
The observable equivalence structure is determined by the projection
operator itself rather than by the scalar functional $O(v)$.
The latter serves only as a low-dimensional summary statistic of the
projected representation $\Pi v$ and does not determine the full
observable geometry.

\begin{figure}[H]
    
    \includegraphics[width=\linewidth]{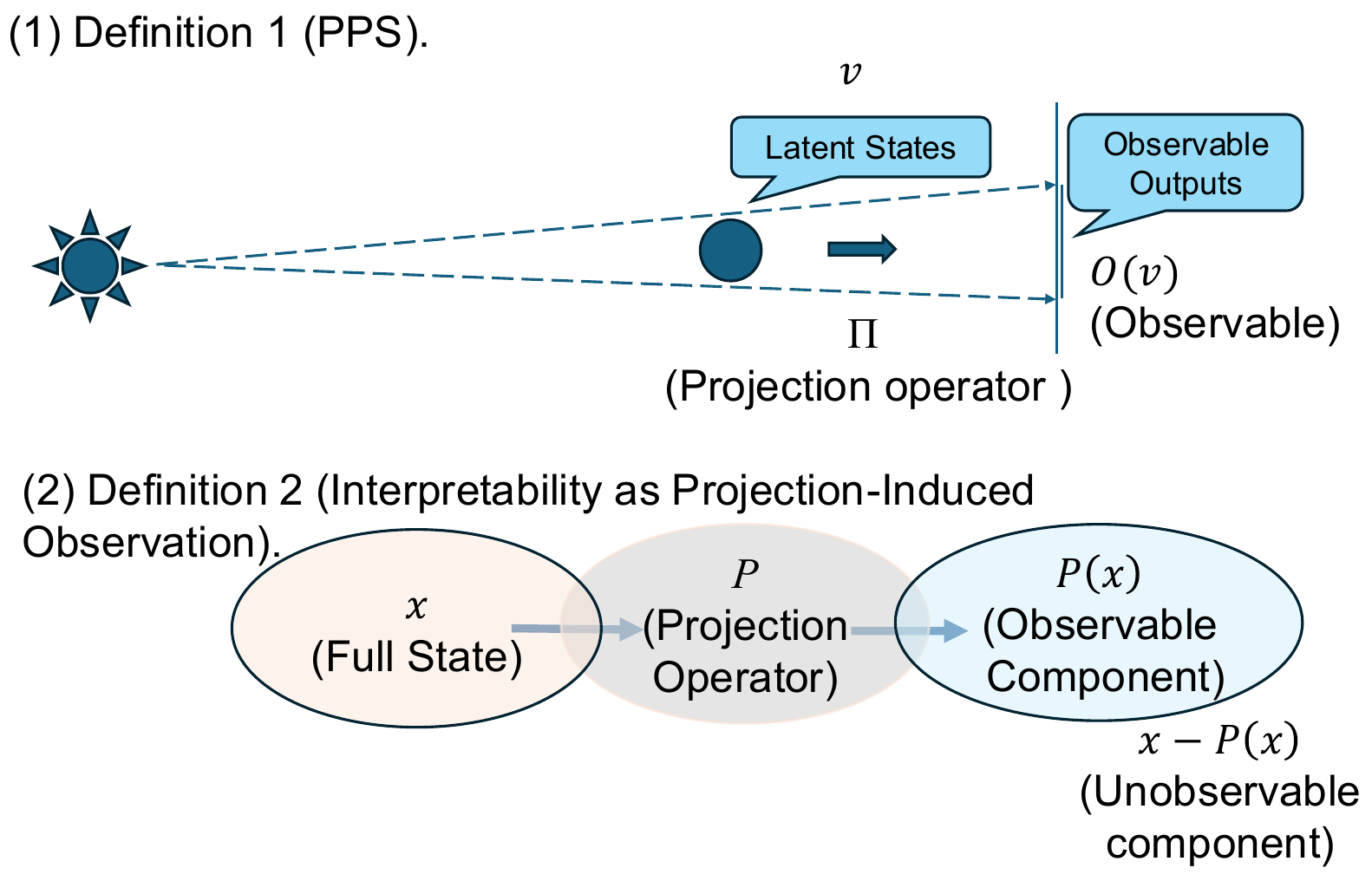}
    \caption{
    {Conceptual
 illustration of observable and unobservable components under projection.}
    The projection operator $P$ maps a latent representation $x$ to its observable component $P(x)$.
    The complementary component $x-P(x)$ lies outside the observable subspace and is therefore inaccessible to output-based explanation methods. 
    }
    \label{fig:interpretability_projection}
\end{figure}

\begin{Proposition}[Observable Equivalence Under Projection]
\label{prop:observable_equivalence}
Let the observable functional be defined as in Equation~\eqref{eq:observable}, where $\Pi$ is a self-adjoint positive semidefinite operator.
If $v_1\sim_\Pi v_2$ under Equation~\eqref{eq:equivalence}, then
\begin{equation}
O(v_1)=O(v_2).
\label{eq:observational_equivalence_result}
\end{equation}
Consequently, any explanation method that operates solely on observables generated through \mbox{Equation~\eqref{eq:observable}} cannot distinguish between latent states belonging to the same equivalence class induced by $\ker(\Pi)$.
\end{Proposition}

\begin{proof}
Let $u=v_1-v_2$ with $u\in\ker(\Pi)$.
Then $\Pi u=0$, and hence $\Pi v_1=\Pi(v_2+u)=\Pi v_2$.
Substituting this into Equation~\eqref{eq:observable} yields Equation~\eqref{eq:observational_equivalence_result}.
Thus, observables cannot distinguish latent states differing only within $\ker(\Pi)$.
\end{proof}
Although algebraically straightforward, Proposition~1 formalizes observational indistinguishability as a 
quotient-geometric property induced by the observation operator, providing the conceptual basis for the 
PPS framework.

\begin{Corollary}[Non-Identifiability of Latent Structure]
\label{cor:nonidentifiable}
Under PPS, latent representations are identifiable only up to the quotient geometry defined in Equation~\eqref{eq:quotient}.
In particular, if $\dim(\ker(\Pi))>0$, then the mapping from latent states to observables is non-injective.
\end{Corollary}
This result shows that output-based interpretability methods cannot reconstruct the full latent representation space.
They can only characterize equivalence classes visible through the induced observation geometry.

\subsection{Structural Interpretability and Attribution Limits}

PPS does not imply that post hoc attribution methods are invalid.
Rather, the framework suggests that such methods are structurally bounded by the induced observation geometry through which observable quantities become accessible.
The PPS framework provides a structural interpretation of post hoc interpretability methods such as SHAP~\cite{shap}, LIME~\cite{lime}, GradCAM~\cite{gradcam}, and integrated gradients~\cite{ig}.
These methods operate on outputs, activations, or gradients that are already constrained by the observation geometry.
Therefore, they characterize only the observable component of the system, not the full latent structure.
If a latent factor is encoded in a direction that lies in or near $\ker(\Pi)$, its contribution to observables generated through Equation~\eqref{eq:observable} is suppressed.
As a result, attribution methods may assign negligible importance to a latent component even when that component remains meaningful within the pre-projection latent process.
This distinction is important because observational sensitivity and intervention sensitivity need not coincide~\cite{pearl,causal_rep}.
A latent component may be unobservable from the output perspective while still affecting downstream behavior through transformations that occur before projection.
Thus, absence of attribution should not be interpreted as absence of latent relevance.
Figure~\ref{fig:pps_ethical} summarizes this structural limitation of post hoc explanation.

\begin{figure}[H]
 
  \includegraphics[width=\linewidth]{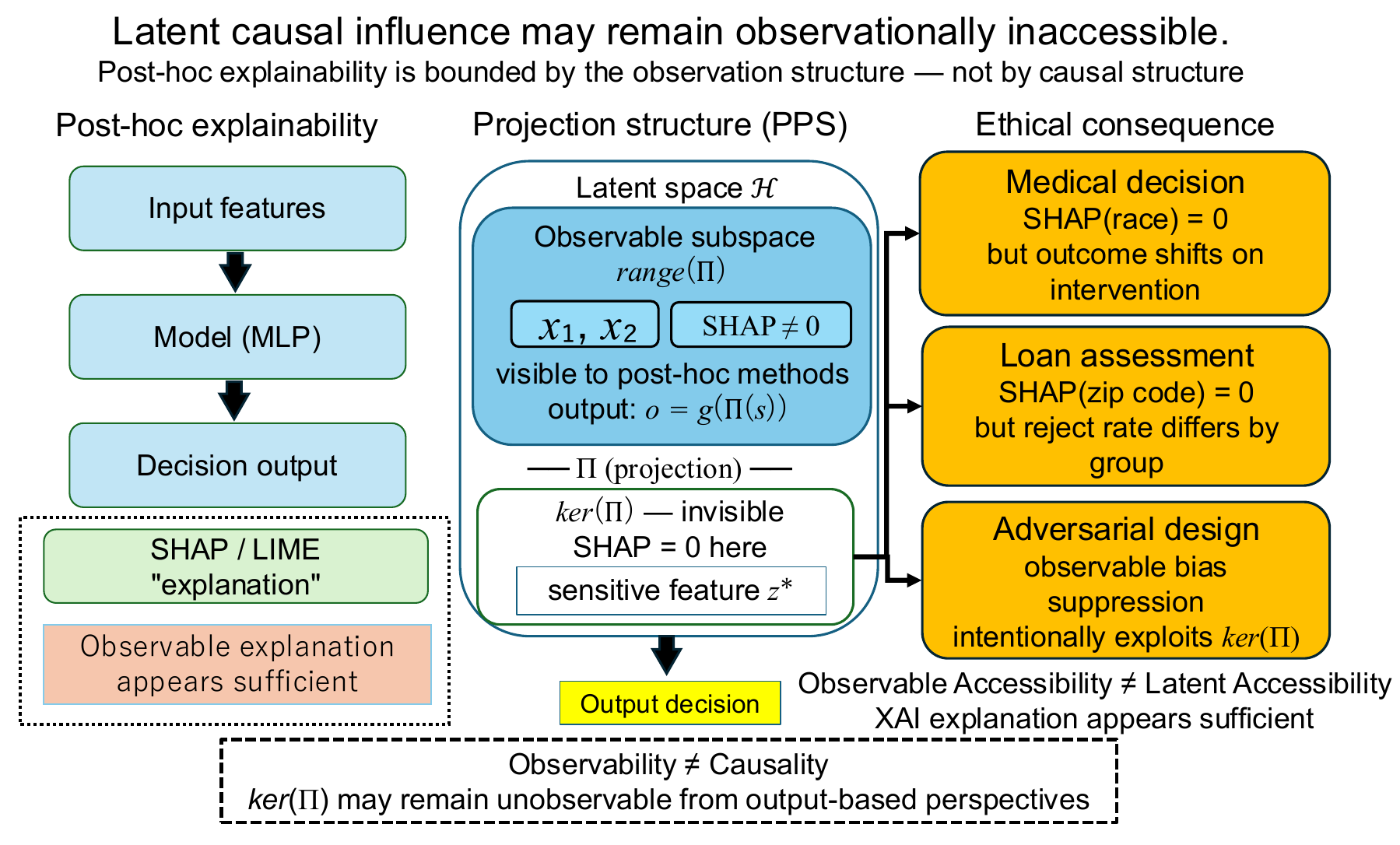}
  \caption{
  {Structural limitations of post hoc explainability under PPS.}
  Post hoc explanation methods operate only on the observable subspace induced by the observation operator.
  Latent components weakly represented within this observable geometry may remain inaccessible to output-based attribution methods such as SHAP, LIME, and GradCAM.
 Consequently, output-based explanations characterize only the observable quotient structure defined in Equation~\eqref{eq:quotient}, rather than the full latent representation space.
  }
  \label{fig:pps_ethical}
\end{figure}
To illustrate this phenomenon empirically, we consider a controlled synthetic setting in which a latent variable $z^\ast$ is placed outside the observable subspace.
When $z^\ast$ is excluded from the observable representation, SHAP assigns negligible attribution to it.
When the same variable is explicitly included in the observable representation, the attribution becomes non-zero.
This change reflects not a change in the underlying latent mechanism, but a change in the accessible observation geometry~\cite{amari,bhatia}.
Importantly, the intervention in this experiment is applied before projection.
Therefore, the experiment does not contradict Proposition~\ref{prop:observable_equivalence}, which concerns invariance under perturbations confined to $\ker(\Pi)$ after the observation geometry has been fixed.
Instead, it demonstrates a separation between latent-space sensitivity and output-based attribution~\cite{pearl,causal_rep}.
This observability--sensitivity gap is illustrated in Figure~\ref{fig:projection_gap}.

\begin{figure}[t]

\includegraphics[width=\linewidth]{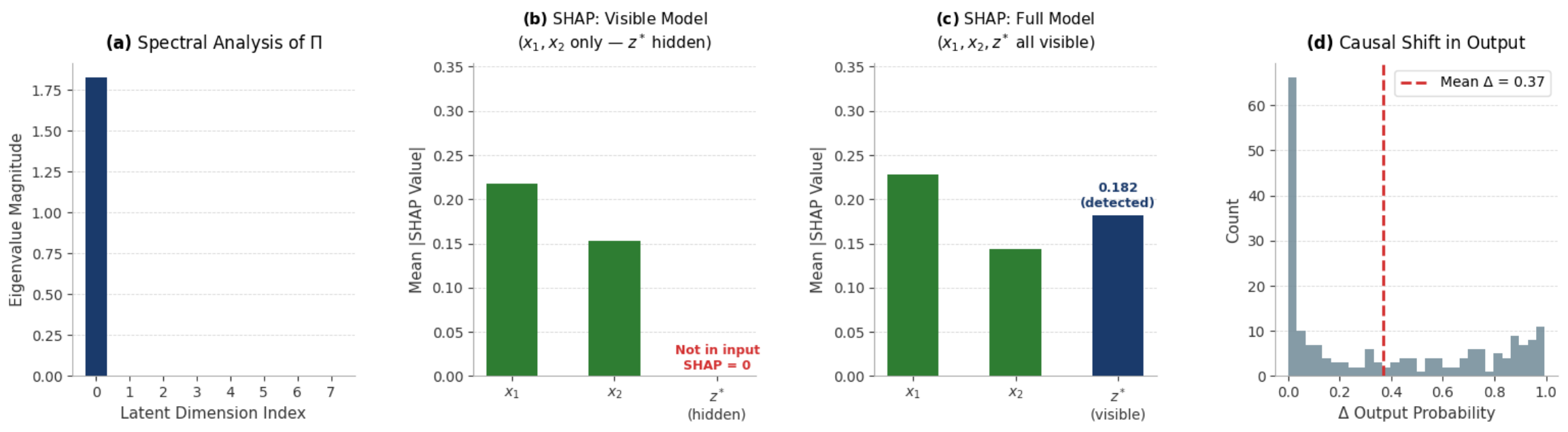}
\caption{
{Empirical
 illustration of the observability--sensitivity gap under PPS.}
(\textbf{a}) Spectral structure of the observation operator showing a non-trivial kernel.
(\textbf{b}) When the latent variable lies outside the observable subspace, SHAP attribution becomes negligible.
(\textbf{c}) When the same variable is included in the observable representation, attribution becomes non-zero.
(\textbf{d}) Latent intervention prior to projection induces measurable output variation despite negligible attribution.
Together, these results show that attribution methods inherit structural limitations imposed by projection-\mbox{mediated observability.}
}
\label{fig:projection_gap}
\end{figure}

\subsection{Toward Structural Accountability}

Projection-induced observability limits suggest that output transparency alone is insufficient for characterizing 
the behavior of representation learning systems. This motivates a shift from output-centric explanation toward 
structural analysis of the observation operator itself.
Within PPS, \emph{structural interpretability} means explicitly characterizing the operator geometry that determines what can become observable.
Correspondingly, \emph{structural accountability} concerns whether the observation operator preserves, suppresses, or distorts application-relevant latent directions.
This perspective leads to several structural questions:
\begin{itemize}
\item Which subspaces are preserved within $\mathrm{range}(\Pi)$?
\item Which latent directions are eliminated by $\ker(\Pi)$?
\item How does the spectral structure in Equation~\eqref{eq:spectral} regulate representational accessibility?
\item How does the effective observable dimension in Equation~\eqref{eq:observable_rank} constrain explanation and prediction?
\end{itemize}
The PPS framework therefore suggests that interpretability constraints can be imposed directly at the operator level.
Representative examples include spectral constraints controlling eigenvalue decay, subspace constraints enforcing preservation of designated latent directions, and operator regularization penalizing suppression of application-relevant components.
Existing approaches can be reinterpreted in this language.
Concept Bottleneck Models~\cite{cbm} impose explicit constraints on intermediate observable subspaces, while invariance-based learning methods may be viewed as inducing constraints on observation geometry~\cite{irm,causal_rep}.
PPS provides a unified formulation in which these methods correspond to different ways of shaping the operator that governs observability.
From this perspective, interpretability is not merely a property of explanations produced after inference.
Rather, it is a property of the observation geometry that determines what can become visible in the first place.

\section{Empirical Verification of PPS}

This section presents controlled experiments designed to directly evaluate
central predictions of the PPS framework. Unlike conventional
performance-oriented evaluations, these experiments are intentionally
constructed as operator-level theory-validation studies that isolate
projection-induced observability effects predicted by PPS.

The experiments focus on three theoretical consequences derived from the
operator-theoretic formulation developed in previous sections:
\begin{enumerate}
\item Invariance of observables under perturbations confined to $\ker(\Pi)$;
\item Dependence of observable behavior on the rank and spectral structure of the observation operator;
\item Approximate preservation of observable geometry through operator-consistent representation transfer.
\end{enumerate}
Accordingly, the empirical objective of this section is not to establish
state-of-the-art performance on large-scale benchmarks, but rather to
examine whether the theoretical predictions of PPS emerge under
controlled experimental conditions.

\subsection{Kernel-Invariant Observability}

We first evaluate the PPS prediction that observables remain invariant under perturbations confined to the kernel of the observation operator~\cite{halmos,reed_simon}.
Consider a latent representation space
$
\mathcal{H}\cong\mathbb{R}^D
$
equipped with a positive semidefinite observation operator
$
\Pi\succeq0
$
of rank
$
R<D.
$
The latent representation is explicitly decomposed into observable and structurally unobservable components:
\begin{equation}
v
=
v_{\mathrm{obs}}
+
v_{\ker},
\label{eq:latent_decomposition}
\end{equation}
where
\[
v_{\mathrm{obs}}\in\mathrm{range}(\Pi),
\qquad
v_{\ker}\in\ker(\Pi).
\]
Observables are generated through the PPS quadratic functional defined previously in Equation~\eqref{eq:observable}.
According to Proposition~\ref{prop:observable_equivalence}, perturbations restricted to $\ker(\Pi)$ should leave observables invariant, whereas perturbations within $\mathrm{range}(\Pi)$ should induce systematic observable variation.
To verify this prediction directly, we constructed an $8$-dimensional latent space with an observation operator of rank $4$.
The eigenspectrum therefore consists of four observable modes and four kernel modes. 
Two perturbation regimes were evaluated:

\begin{enumerate}
\item \textbf{Kernel perturbation:
}
\begin{equation}
v
\rightarrow
v+\delta_{\ker},
\qquad
\delta_{\ker}\in\ker(\Pi),
\label{eq:kernel_perturbation}
\end{equation}

\item \textbf{Observable perturbation:
}
\begin{equation}
v
\rightarrow
v+\delta_{\mathrm{obs}},
\qquad
\delta_{\mathrm{obs}}\in\mathrm{range}(\Pi),
\label{eq:observable_perturbation}
\end{equation}
\end{enumerate}

Figure~\ref{fig:kernel_invariance}a shows the eigenspectrum of the observation operator, separating observable and kernel directions.
Figure~\ref{fig:kernel_invariance}b demonstrates that perturbations generated through Equation~\eqref{eq:kernel_perturbation} leave observables numerically invariant up to machine precision.
In contrast, Figure~\ref{fig:kernel_invariance}c shows that perturbations generated through Equation~\eqref{eq:observable_perturbation} induce substantial observable variation with approximately quadratic dependence on perturbation magnitude.
This provides a direct empirical visualization of the PPS prediction that latent directions contained in $\ker(\Pi)$ are structurally inaccessible under the induced observation geometry. 
Finally, Figure~\ref{fig:kernel_invariance}d compares SHAP attribution behavior between restricted and unrestricted models.
When the model is explicitly constrained to $\mathrm{range}(\Pi)$, kernel dimensions receive zero attribution by construction.
In contrast, unrestricted models distribute attribution across both observable and kernel components.
This suggests that attribution methods may reflect representational leakage outside the intended 
observable subspace structure induced by the operator.

These attribution experiments are intended as illustrative demonstrations of projection-induced observability limits 
rather than benchmark evaluations of attribution performance. The objective is to visualize the qualitative
effect predicted by PPS, namely that attribution methods can only explain components that remain accessible 
through the induced observation geometry. Accordingly, statistical significance testing across multiple random 
seeds was not the primary focus of these controlled theory-validation experiments.

\begin{figure}[t]
\centering
\includegraphics[width=.80\linewidth]{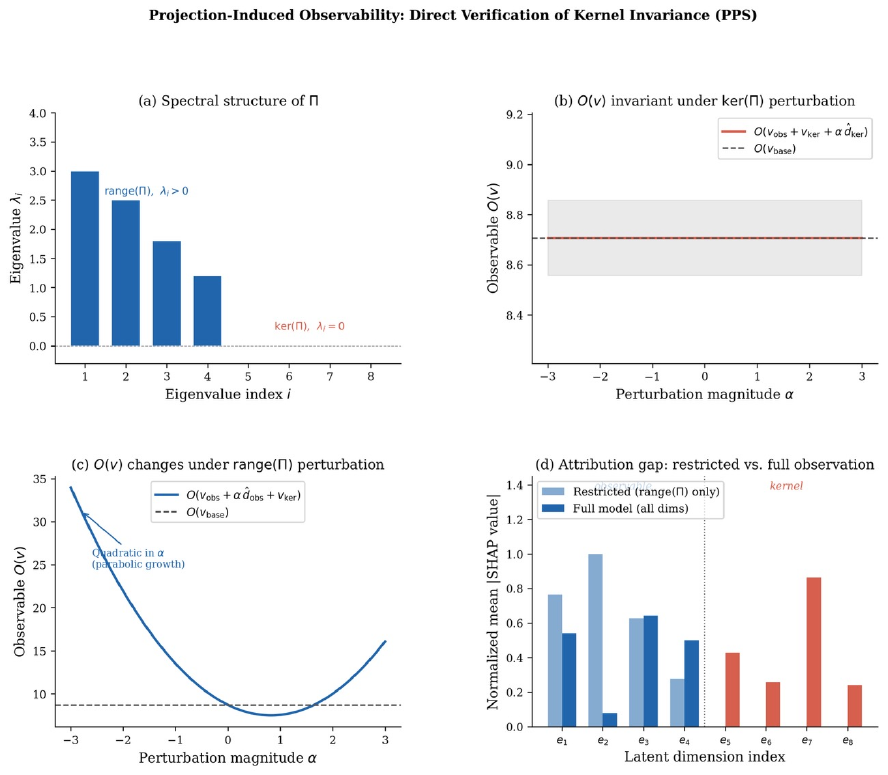}
\caption{
{Direct
 verification of kernel-invariant observability under PPS.}
(\textbf{a}) Eigenvalue spectrum of the observation operator showing decomposition into observable and kernel directions.
\mbox{(\textbf{b}) Observable} quantities remain invariant under perturbations confined to $\ker(\Pi)$, directly validating Proposition~\ref{prop:observable_equivalence}.
(\textbf{c}) Perturbations within $\mathrm{range}(\Pi)$ induce systematic observable variation.
\mbox{(\textbf{d}) SHAP} attribution comparison between restricted and unrestricted models.
Restricted models assign zero attribution to kernel dimensions, whereas unrestricted models do not.
}
\label{fig:kernel_invariance}
\end{figure}

\subsection{Rank-Controlled Observable Geometry}

We next investigate how the rank structure of the observation operator governs effective observable geometry~\cite{spectral_ml,jolliffe,lowrank}. Within PPS, the effective observable dimension is determined by Equation~\eqref{eq:observable_rank}. Consequently, observable behavior should depend not on the ambient latent dimensionality itself, but on the rank and spectral structure of the observation operator.
To evaluate this prediction, we generated a synthetic classification task in a $16$-dimensional latent space where the true signal was confined to an intrinsic subspace of dimension
$
r^\ast=8.
$
Observation operators
$
\Pi_r
$
with varying rank
$
r\in\{1,\dots,16\}
$
were constructed by progressively truncating the eigenspectrum introduced in Equation~\eqref{eq:spectral}. 
For each rank setting, we evaluated

\begin{enumerate}
\item Classification accuracy;
\item Attribution leakage into $\ker(\Pi)$;
\item Cumulative observable spectral energy.
\end{enumerate}
For empirical analyses involving rank estimation, numerical rank was computed using an eigenvalue 
threshold of $\varepsilon = 10^{-6}$, with eigenvalues below this threshold treated as numerically zero.

Figure~\ref{fig:rank_geometry}a shows that predictive performance increases systematically with $\mathrm{rank}(\Pi)$ and saturates near the intrinsic signal rank
$
r^\ast=8.
$ Accuracy increases from approximately $0.61$ at rank $1$ to approximately $0.94$ at rank $8$, after which additional observable dimensions produce minimal improvement.
Figure~\ref{fig:rank_geometry}b shows the normalized kernel-attribution leakage ratio, defined as the proportion of attribution mass assigned to dimensions outside the observable subspace. Leakage decreases monotonically as $\mathrm{rank}(\Pi)$ increases, indicating that progressively more signal-bearing directions become observable.
Figure~\ref{fig:rank_geometry}c shows the cumulative observable spectral-energy ratio captured by the induced observation geometry.
At $r^\ast = 8$, the observable subspace captures nearly all signal-associated spectral energy, consistent with the observed performance saturation. 
Finally, Figure~\ref{fig:rank_geometry}d visualizes cumulative spectral-energy accumulation profiles for multiple rank settings, illustrating how eigenspectrum truncation modifies effective observable geometry.
Collectively, these results support the PPS prediction that observable behavior is governed by the spectral structure of the observation operator rather than by latent dimensionality alone.
\begin{figure}[H]

\centering 
\includegraphics[width=.88\linewidth]{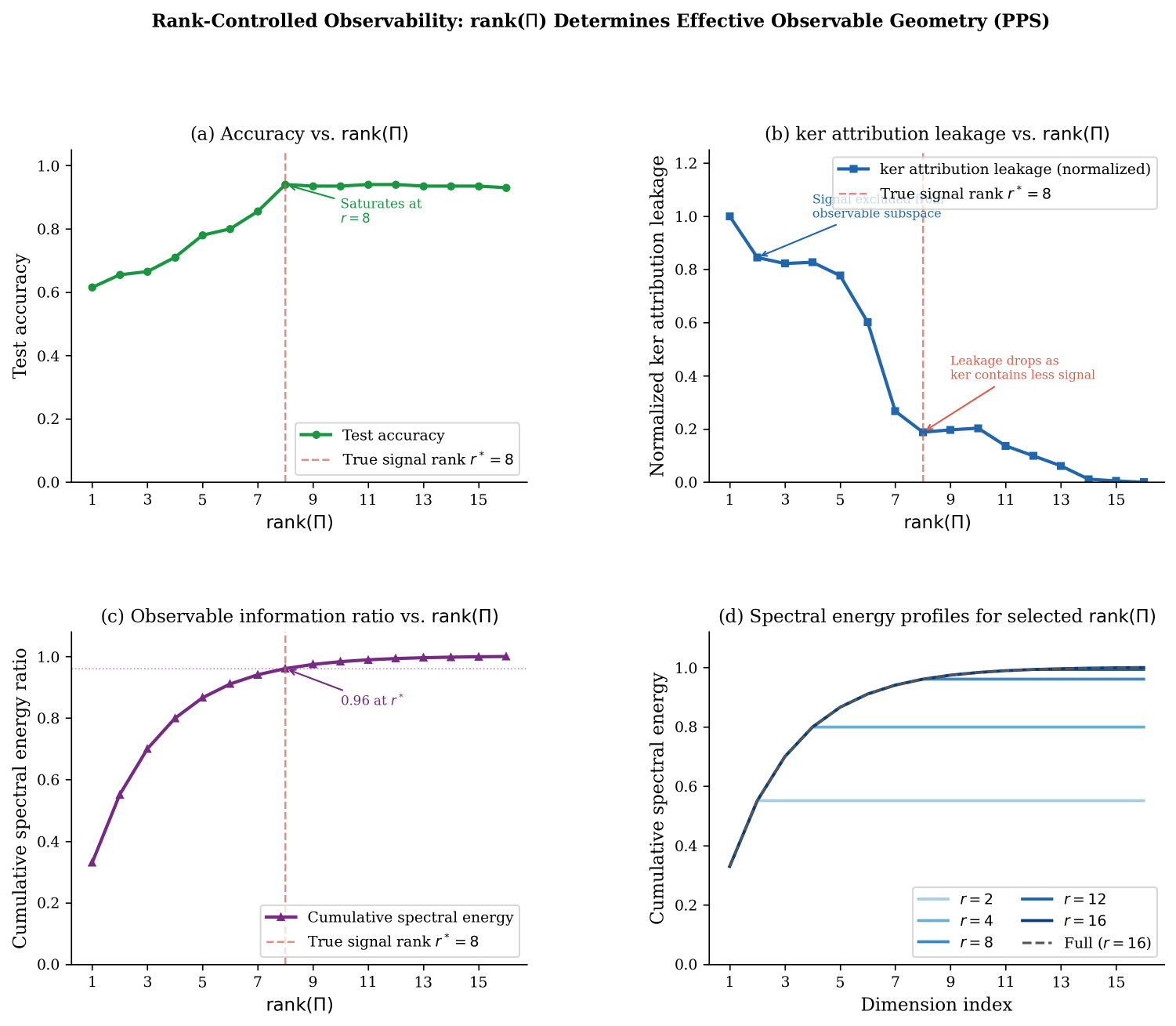}

\caption{
{Rank-controlled
 observable geometry under PPS.}
(\textbf{a}) Classification accuracy as a function of $\mathrm{rank}(\Pi)$.
Performance increases systematically with observable dimension and saturates near the intrinsic signal rank.
(\textbf{b}) Attribution leakage into $\ker(\Pi)$ decreases as the observable subspace expands.
(\textbf{c}) Cumulative observable spectral-energy ratio captured by the induced observation geometry.
(\textbf{d}) Spectral-energy accumulation profiles for multiple rank settings, illustrating how eigenspectrum truncation modifies effective observable structure.
}
\label{fig:rank_geometry}
\end{figure}

\subsection{Operator-Consistent Knowledge Distillation}

Finally, we evaluate the PPS interpretation of representation transfer introduced in Equation~\eqref{eq:intertwining}~\cite{hinton2015,rkd,crd}. 
The experiment compares standard KL-based knowledge distillation~\cite{hinton2015} with a PPS-regularized formulation incorporating the operator-consistency objective defined in Equation~\eqref{eq:pps_loss}. A ResNet18 teacher model and a lightweight CNN student model were trained on CIFAR-10. The interspace map introduced in Equation~\eqref{eq:interspace_map} was implemented as a trainable linear projection jointly optimized with the student network.

Figure~\ref{fig:projection_noncommutativity} summarizes the comparison between standard distillation and PPS-regularized distillation. Panel (a) shows the evolution of the commutativity gap induced by \mbox{Equation~\eqref{eq:pps_loss},} panel (b) reports predictive accuracy, and panel (c) shows the final operator inconsistency after training.
The PPS-regularized formulation consistently reduces operator inconsistency between teacher and student systems while maintaining comparable predictive accuracy. Importantly, the improvement occurs at the level of induced observation geometry rather than solely at the level of output agreement. These results support the PPS interpretation that successful representation transfer may depend not only on observable output matching, but also on approximate compatibility between the corresponding observation operators. Rather than serving as a benchmark-oriented evaluation, the experiment provides a controlled empirical verification that operator-induced geometry can be treated as an explicit and measurable object within representation-transfer pipelines.

\section{Discussion}

The PPS framework provides an operator-theoretic perspective on observability in representation learning systems~\cite{halmos,reed_simon,bhatia,bengio_rep}.
Rather than treating observable outputs as direct reflections of latent representations, PPS models observation as geometry induced through self-adjoint positive semidefinite operators acting on latent Hilbert spaces~\cite{halmos,bhatia,amari}.
This viewpoint connects several previously separate research directions, including information bottleneck theory, observability theory, representation learning, and interpretability analysis~\cite{tishby_ib,kalman,bengio_rep,doshi,rudin}.

\subsection{Relation to Information Bottleneck and Representation Learning}

PPS is conceptually related to the information bottleneck (IB) principle~\cite{shannon,tishby_ib,alei_vib,mine}, which studies compressed representations through mutual-information constraints.
Both frameworks emphasize that learned systems preserve only selected aspects of latent information.
However, the two approaches differ fundamentally in formulation.
Information bottleneck methods characterize compression probabilistically through random variables and mutual-information objectives~\cite{shannon,tishby_ib,alei_vib,mine},
whereas PPS characterizes observability geometrically through operator-induced quotient geometry and spectral structure.
Within PPS, information accessibility is governed by the quotient geometry introduced in \mbox{Equation~\eqref{eq:quotient}} and the eigenspectrum defined in \mbox{Equation~\eqref{eq:spectral}.}
Observable information therefore becomes a property of induced operator geometry rather than solely of probabilistic compression.
This perspective also relates naturally to modern representation learning. 
Many learning systems implicitly construct observable subspaces through learned readout operators, as exemplified by Equation~\eqref{eq:wtw}~\cite{bengio_rep,scholkopf_smola}.
From the PPS viewpoint, latent representations should therefore be interpreted together with the observation geometry through which they become accessible.

\subsection{Relation to Observability Theory}

PPS also exhibits structural parallels with classical observability theory in control systems~\cite{kalman}.
Classical observability theory studies whether internal dynamical states can be reconstructed from observable outputs over time.
By contrast, PPS focuses on structural accessibility induced by observation operators acting on latent representation spaces~\cite{halmos,reed_simon,bhatia}.
Despite this distinction, both share a common conceptual principle:
observable behavior depends not only on latent states themselves, but also on the structure governing access to those states.
Within PPS, this structure is encoded through the operator geometry induced by $\Pi$ and the quotient structure defined in Equation~\eqref{eq:quotient}.

\subsection{Interpretability and Structural Accessibility}

A central implication of PPS concerns the structural limits of output-based interpretability~\cite{doshi,rudin}.
Since observables generated through Equation~\eqref{eq:observable} are invariant under perturbations in $\ker(\Pi)$, they depend only on the quotient structure defined in \mbox{Equation~\eqref{eq:quotient}.}
The experiments presented in Sections~5 and~6 directly illustrate this phenomenon.
Kernel-confined perturbations remain observationally invisible despite latent modification, while eigenspectrum truncation systematically alters attribution behavior and effective observable dimensionality.
These results suggest that attribution maps and output-based explanations should not be interpreted as direct reconstructions of latent structure.
Rather, they characterize only the observable geometry induced through the observation operator.
From this viewpoint, interpretability limitations are not merely algorithmic deficiencies, but structural consequences of projection-mediated observability.

\subsection{Relation to the Platonic Representation Hypothesis}

Recent discussions surrounding the Platonic Representation Hypothesis~\cite{platonic_representation} suggest that independently trained neural systems may converge toward structurally similar latent representations.
PPS provides a complementary perspective on this idea.
Rather than focusing solely on latent representation similarity, PPS emphasizes the geometry induced by observation operators acting on those representations.
From this viewpoint, similarity between systems may arise not only from alignment of latent states themselves, but also from compatibility between induced observable geometries.
The operator-consistency relation introduced in Equation~\eqref{eq:intertwining} provides one possible formalization of this perspective.
The knowledge-distillation experiments further suggest that approximate preservation of operator geometry may serve as a measurable notion of structural compatibility between representation systems~\cite{hinton2015,rkd,crd,attention_transfer}. 
More broadly, PPS suggests that observability may provide a useful geometric language for relating representation learning, interpretability, information accessibility, and projection-mediated inference across diverse systems~\cite{amari,bengio_rep,bhatia}.
In this context, the term ``Platonic'' is used structurally to emphasize the distinction between latent structure and observable projection, rather than to invoke a metaphysical claim.

\subsection{Limitations and Future Work}

Several limitations of the present work should be emphasized.

First, the current PPS formulation primarily analyzes static observation geometry.
Extensions to temporally evolving operators and sequential systems remain future work.

Second, the empirical experiments are intentionally controlled and illustrative rather than benchmark-oriented.
The primary goal is conceptual validation of projection-mediated observability rather than state-of-the-art predictive performance.

Third, the present framework focuses primarily on linear self-adjoint positive semidefinite observation operators~\cite{bhatia,reed_simon}.
Extensions to nonlinear operators, stochastic observation processes, non-self-adjoint operators, non-Hilbert latent geometries, and temporally evolving observability remain important directions for future work.

\begin{itemize}
\item Operator-constrained representation learning;
\item Spectral regularization for interpretable latent geometry;
\item Projection-aware explanation methods;
\item Geometric formulations of transfer learning;
\item Extensions to sequential and multi-modal systems.
\end{itemize}

While the present study focuses on controlled synthetic experiments and
a CIFAR-10 distillation setting, broader empirical validation remains an
important direction for future work. In particular, evaluating PPS on
large-scale vision benchmarks, natural-language representation models,
medical imaging datasets, and foundation-model representations would
provide additional evidence regarding the practical applicability of the
framework beyond the controlled settings considered here.

For a nonlinear readout
\[
y=f(z),
\]
a possible extension of PPS can be obtained through
local linearization around a reference representation
\(z_0\):
\[
f(z)\approx f(z_0)+J_f(z_0)(z-z_0),
\]
where \(J_f(z_0)\) denotes the Jacobian of the readout.
This induces a local observation operator
\[
\Pi_{\mathrm{local}}
=
J_f(z_0)^T J_f(z_0).
\]
Under this formulation, observability becomes
state-dependent, yielding a local observable geometry
that may vary across the representation manifold.
Related approximations based on empirical Fisher
information matrices may provide an alternative route
toward nonlinear observability analysis.
Experimental validation of Jacobian- or Fisher-based
observation operators requires substantially different
benchmark settings and is therefore left for \mbox{future work.}

Another important direction concerns attribution analysis
under controlled latent factors and spurious correlations.
Within PPS, explanatory signals associated with latent
variables lying in \(\ker(\Pi)\) are expected to be
suppressed or absent from observable outputs, whereas
observable latent directions may contribute strongly to
attribution scores.

Empirical validation of this prediction on real-world
datasets remains an important topic for future work.
More broadly, future work should investigate whether
PPS-based representations improve robustness under
out-of-distribution conditions or facilitate transfer
learning across related tasks. Such evaluations would
provide further evidence regarding the practical
consequences of operator consistency.
\subsection{Relationship to Existing Frameworks}

The relationship between PPS and information bottleneck methods has been discussed above.
Here, we compare PPS more broadly with existing representation-analysis frameworks, including
kernel PCA, canonical correlation analysis (CCA), and information bottleneck formulations.
Although PPS shares certain structural similarities with these approaches, it addresses a
fundamentally different question: which latent directions become observable under a given observation operator.

Kernel PCA and related spectral methods~\cite{scholkopf_smola,jolliffe} analyze latent representations through 
spectral decomposition of kernel-induced operators. Their primary objective is dimensionality reduction 
and extraction of informative low-dimensional structure. In contrast, PPS does not seek dimensionality reduction 
itself. Rather, PPS characterizes which latent directions are observable under a given observation operator and 
formalizes observability through the quotient geometry $H/\ker(\Pi)$.

Similarly, canonical correlation analysis (CCA) and representation-alignment methods~\cite{hotelling1992}
seek correlated subspaces between multiple representations. These approaches are widely used for comparing 
neural representations and analyzing transferability across models. PPS addresses a complementary problem: 
instead of maximizing correspondence between representations, it characterizes the observable geometry induced 
by the observation operator itself. From the PPS perspective, representation-alignment methods may be 
interpreted as implicitly constructing mappings between observable quotient structures.

Information bottleneck formulations~\cite{tishby_ib,alei_vib} provide an information-theoretic framework in 
which representations are optimized through a trade-off between compression and predictive relevance. 
PPS differs in that it does not explicitly quantify information content through mutual information. 
Instead, it characterizes representation accessibility through operator-induced observability. 
In this sense, information bottleneck methods focus on how much information is retained, whereas PPS focuses on 
which latent directions can become observable.

These perspectives are therefore complementary rather than competing. Kernel methods, CCA, and information bottleneck 
approaches may all be viewed as constructing, constraining, or analyzing operators acting on latent representations. 
PPS provides an operator-theoretic framework for analyzing the observable geometry induced by such operators and the 
resulting limits of representation accessibility. Consequently, PPS should be viewed not as a replacement for these
frameworks, but as an operator-theoretic perspective that clarifies the geometric conditions under which latent information 
becomes observable. 

\section{Conclusions}

This paper introduced Platonic Projection Structures (PPS), an operator-theoretic framework for analyzing projection-mediated observability in representation learning systems.
Within PPS, observable quantities are generated not through direct access to latent representations, but through geometry induced by self-adjoint positive semidefinite observation operators.
The framework formalizes this perspective through the observable functional defined in Equation~\eqref{eq:observable}, the quotient structure introduced in Equation~\eqref{eq:quotient}, and the spectral decomposition in Equation~\eqref{eq:spectral}.
From this viewpoint, observability becomes a structural property of induced operator geometry rather than a property of latent states themselves.
The framework further provides a unified abstraction spanning quantum measurement, deep learning inference, representation transfer, and interpretability analysis~\cite{vonneumann,nielsen2000,bengio_rep,hinton2015}.
Quantum systems and neural systems were shown to admit structurally parallel formulations as projection-mediated observation processes, despite substantial differences in physical interpretation~\cite{vonneumann,biamonte,schuld}.
PPS also provides a geometric interpretation of representation transfer through the operator-consistency relation introduced in Equation~\eqref{eq:intertwining}.
The corresponding experiments demonstrated that operator-level consistency can be incorporated explicitly within knowledge-distillation pipelines while preserving predictive performance.
A central implication of PPS concerns the structural limits of output-based interpretability.
Because observable quantities depend only on the induced observable quotient geometry, attribution methods operating exclusively on outputs necessarily inherit constraints imposed by the observation operator itself.
The controlled experiments presented in this work verified several central predictions of the \mbox{framework, including}
\begin{enumerate}
\item Kernel-invariant observability;
\item Rank-controlled observable geometry;
\item Operator-consistent representation transfer.
\end{enumerate}
Collectively, these results support the PPS interpretation that observable behavior is governed by operator-induced geometry rather than by latent representations alone.
More broadly, PPS suggests that observability may serve as a unifying geometric principle connecting representation learning, interpretability, information accessibility, and projection-mediated inference across diverse systems.
Future work will investigate extensions to nonlinear operators, sequential systems, adaptive observation geometry, multi-modal architectures, and projection-aware interpretability methods.
Ultimately, the PPS framework proposes a shift in perspective:
rather than asking only what latent representations contain, one may instead ask what the induced observation geometry allows a system to observe.







\section*{Author Contributions}
Conceptualization, K.I. and B.P.G.; methodology, K.I.; software, B.P.G.; validation, J.W. and J.S.; formal analysis, K.I.; resources, B.P.G. and J.S.; data curation, J.W.; writing---original draft preparation, K.I. and B.P.G.; writing---review and editing, K.I., J.W. and J.S.; visualization, B.P.G.; supervision, K.I.; project administration, K.I.; funding acquisition, K.I. All authors have read and agreed to the published version of the manuscript.

\section*{Funding}
This research received no external funding. The APC was funded by Suwa University of Science.

\section*{Data Availability Statement}
The synthetic data used in this study can be reproduced using the procedures described in the manuscript. The code used for the synthetic experiments, CIFAR-10 distillation experiments, and synthetic-data generation procedures will be made publicly available through a GitHub repository upon acceptance of the manuscript.

\section*{Conflicts of Interest}
The authors declare no conflicts of interest.

\section*{Abbreviations}
\begin{tabular}{@{}ll}
PPS  & Platonic Projection Structures \\
KL   & Kullback--Leibler \\
KD   & Knowledge Distillation \\
IB   & Information Bottleneck \\
CNN  & Convolutional Neural Network \\
PSD  & Positive Semidefinite \\
SHAP & SHapley Additive exPlanations \\
XAI  & Explainable Artificial Intelligence \\
\end{tabular}

\section*{References}

\section*{References}


\end{document}